\documentclass[12pt]{article}
\RequirePackage{kpfonts}
\RequirePackage[sf,bf,small,raggedright]{titlesec}
\usepackage{hyperref}
\hypersetup{
     colorlinks   = true,
     citecolor    = blue
}

\usepackage{dsfont}
\usepackage{comment}
\usepackage{textcomp}
\usepackage{amssymb}
\usepackage{bbm}
\usepackage{fancyhdr}
\usepackage{amsmath}
\usepackage{amsbsy,amsthm}
\usepackage{amscd}
\usepackage{latexsym}
\usepackage{graphicx}   
\usepackage{pdfsync}
\usepackage{blkarray}
\usepackage{multirow}
\usepackage{hyperref}
\usepackage{xcolor}
\usepackage{enumerate}
\usepackage{tcolorbox}
\usepackage{subcaption}
\usepackage[utf8]{inputenc}
\usepackage[export]{adjustbox}
\usepackage{wrapfig}
\usepackage[parfill]{parskip}
\usepackage{bm}

\usepackage{latexsym}
\usepackage{amsmath,amsfonts,amssymb,mathrsfs,color}
\usepackage{amsthm}
\usepackage{titlesec}
\usepackage{times}
\usepackage[margin=2.6cm]{geometry}

\usepackage[utf8]{inputenc} 
\usepackage[T1]{fontenc}    
\usepackage{hyperref}       
\usepackage{url}            
\usepackage{booktabs}       
\usepackage{amsfonts}       
\usepackage{nicefrac}       
\usepackage{microtype}      
\usepackage{amsmath,amssymb,mathrsfs,amsthm}
\usepackage{times}
\usepackage{algorithm}
\usepackage[noend]{algpseudocode}

\newcommand{\doublehat}[1]{\widehat{#1}}

\newcommand{\E}{\mathbb E}
\newcommand{\F}{\mathcal F}

\newcommand{\A}{\mathcal{A}}
\renewcommand{\L}{\mathcal{L}}
\renewcommand{\O}{\mathcal{O}}
\newcommand{\B}{\mathcal{B}}
\newcommand{\N}{\mathbb{N}}
\newcommand{\G}{\mathcal{G}}

\newcommand{\x}{\mathbf{x}}
\newcommand{\norm}[1]{\left\lVert#1\right\rVert}

\newcommand{\R}{\mathbb{R}}
\newcommand{\abs}[1]{\left\lvert#1\right\rvert}

\renewcommand{\Theta}{\varTheta}
\newcommand{\wh}{\widehat}
\renewcommand{\tilde}{\widetilde}
\newcommand{\one}{\mathbb{I}}

\DeclareMathOperator*{\argmin}{\arg\!\min}
\DeclareMathOperator*{\argmax}{\arg\!\max}
\DeclareMathOperator*{\polylog}{polylog}

\newtheorem{theorem}{Theorem}
\newtheorem{proposition}{Proposition}

\renewcommand{\phi}{\varphi}

\title{\sc Restless Linear Bandits}
\author{Azadeh Khaleghi\\ENSAE - CREST, IP Paris }
\date{~}

\begin{document}
\maketitle

\begin{abstract}
A more general formulation of the linear bandit problem is considered to allow for dependencies over time.
Specifically, it is assumed that there exists an unknown $\R^d$-valued stationary $\varphi$-mixing sequence of parameters $(\theta_t,~t \in \N)$ which gives rise to pay-offs. This instance of the problem can be viewed as a generalization of both the classical linear bandits with iid noise, and the finite-armed restless bandits. In light of the well-known computational hardness of optimal policies for restless bandits, an approximation is proposed whose error is shown to be controlled by the $\varphi$-dependence between consecutive $\theta_t$. An optimistic algorithm, called LinMix-UCB, is proposed for the case where $\theta_t$ has an exponential mixing rate. The proposed algorithm is shown to incur a sub-linear regret of $\O\left(\sqrt{d n\polylog(n) }\right)$ with respect to an oracle that always plays a multiple of  $\E\theta_t$. The main challenge in this setting is to ensure that the exploration-exploitation strategy is robust against long-range dependencies. 
  The proposed method relies on Berbee's coupling lemma to carefully select near-independent samples and construct
  confidence ellipsoids around empirical estimates of $\E\theta_t$.
\end{abstract}

\section{\sc \bfseries Introduction}
The problem of sequential decision making in the presence of uncertainty is prevalent in a variety of modern applications such as online advertisement, recommendation systems,  network routing and dynamic pricing. Through the framework of {\em stochastic bandits} one can model this task as a repeated game between a {\em learner} and a stochastic {\em environment}: at every time-step, the learner chooses an {\em action} from a pre-specified set of actions $\A$ and receives a (random) real-valued pay-off. The objective is to maximize the expected cumulative pay-off over time. 
In a stochastic {\em linear} bandit model the pay-off $Y_t$ received at each time-step $t$ is the inner product between an $\R^d$-valued action $X_t$ and an unknown parameter vector $\theta \in \R^d$. That is, 
$Y_t = \langle \theta, X_t \rangle + \eta_t $ with random noise $\eta_t$; see e.g. \cite{A02,ADC11}, as we well as \cite{BC12, LS20} and references therein. Let us point out that in the case where the action space $\A$ is restricted to a set of standard unit vectors in $\R^d$ the problem is reduced to that of finite-armed bandits. 
The noise sequence $\eta_t$ is typically assumed to be independently and identically distributed (iid). 
However, this assumption does not usually hold in practice.

In this paper, we consider a more general formulation of the linear bandit problem which allows for dependencies over time. 
More specifically, we assume that there exists an unknown $\R^d$-valued stationary sequence of parameters $(\theta_t,~t \in \N)$ giving rise to the pay-offs 
$Y_t = \langle \theta_t,X_t \rangle,~t \in \N$ with the actions $X_t$ taking values in a unit ball in $\R^d$.  
As with any bandit problem, the task of balancing the trade-off between exploration and exploitation calls for finite-time analysis, which in turn relies on concentration inequalities for empirical averages. For this reason, we opt for a natural approach in time-series analysis which is to require that the process satisfy a form of asymptotic independence.
More specifically, we assume that 
$(\theta_t,~t \in \N)$ is $\phi$-mixing so that its sequence of $\phi$-mixing  coefficients $\varphi_m,~m \in \N$ defined by 
\begin{equation*}
\varphi_m:=\sup_{j \in \N}\sup_{\substack{U \in \sigma(\{\theta_t:t=1,\dots,j\})\\ 
V \in \sigma(\{\theta_t:t \geq j+m\}) }}\left |P(V)-P(V|U)\right| 
\end{equation*}
converges to $0$ as $m$ approaches infinity, see, e.g. \cite{BRA07, RIO17}. 
Observe that while $\theta_t$ are identically distributed here, they are not independent. 

To compare with the classical formulation, we can write $Y_t = \langle \theta^*,X_t \rangle + \eta_t $ where $\theta^* := \E \theta_t$ is the (unknown) stationary mean of $\theta_t$ and $\eta_t:=\langle \theta_t-\theta^*,X_t \rangle$ is the noise process, which is clearly non-iid.
This instance of the problem can be viewed as a generalization of the classical linear bandit problem with iid $\eta_t$ \cite{ADC11} and the finite-armed restless Markovian and $\varphi$-mixing bandits considered in \cite{ORAM14} 
and \cite{GK19} respectively. 
Observe that in much the same way as with the examples given in \cite[Section 3]{ORAM14}, a strategy that leverages temporal dependencies can accumulate a higher expected pay-off than that which can be obtained by playing a fixed action in $\A$ that is most aligned with $\theta^*$. In the finite-armed restless bandit problem, this is called the {\em optimal  switching strategy} whose exact computation in certain related instances of the problem is known to be intractable \cite{PT99}. In \cite{CGD23}, an algorithm is proposed for finite-armed restless bandits in the case where the pay-off distributions follow an AR model.
An approximation of the optimal switching strategy for finite-armed restless $\varphi$-mixing bandits is provided in \cite{GK19}. 

We build upon \cite{GK19} to approximate the optimal restless bandit strategy using the Optimism in the Face of Uncertainty principal. We show in Proposition~\ref{prop:oraclediff} that the approximation error of an oracle which always plays a multiple of $\theta^*$, is controlled by $\varphi_1$ and the $\ell_2$ norm of $\theta_t$, provided that the latter is almost surely bounded. The proof relies on  Vitali's covering theorem \cite{M99} to address the technical challenges presented by the infinite action space $\A \subseteq \R^d$. We propose an algorithm, namely LinMix-UCB, which aims to minimize the regret in the case where the process $(\theta_t,~t \in \N)$  has an exponential mixing rate. The proposed algorithm is inspired by the UCB-based approach of
 \cite{ADC11} which is in turn 
 designed for the simpler setting where the noise $\eta_t$ is iid. 
 The main challenge in our setting is to devise an exploration-exploitation strategy that is robust against the long-range dependencies  present in the process $(\theta_t,~t \in \N)$. 
  We rely on Berbee's coupling lemma \cite{BER79} to carefully select near-independent samples and construct
  confidence ellipsoids around empirical estimates of $\theta^*$. We demonstrate that LinMix-UCB incurs a sub-linear regret of $\O\left(\sqrt{d n\polylog(n) }\right)$ with respect to an oracle which always plays a multiple of $\theta^*$; this is shown in Theorems~\ref{thm:regfh}~and~\ref{thm:regih} for the finite and infinite horizon settings respectively. 
 
\section{\sc \bfseries Preliminaries and Problem Formulation}\label{sec:pre}
Let $\Theta \subseteq \R^d$ for some $d \in \N$. Suppose that $\boldsymbol{\theta}:=(\theta_t,~t \in \N)$ is a stationary sequence of $\Theta$-valued random variables defined on a probability space $(\Omega, \B, P)$ such that $\theta_t,~t \in \N$ takes values in the space $\L_{\infty}(\Omega, \B, P; \R^d)$ equipped with its Euclidean norm $\norm{\cdot}_2$. This means that $\norm{\theta_t}_{\L_{\infty}}:=\sup_{\omega \in \Omega} \norm{\theta_t(\omega)}_2 < \infty$. We may sometimes use $\L_{\infty}$ or $\L_{\infty}(\Omega;\R^d)$ for $\L_{\infty}(\Omega,\B,P;\R^d)$ when the remaining parameters are clear from the context. 
Recall that the $\varphi$-dependence (see, e.g. \cite{BRA07}) between any pair of $\sigma$-subalgebras $\mathcal U$ and $\mathcal V$ of $\B$ is given by 
$
\varphi(\mathcal U,\mathcal V) := \sup\{\abs{P(V) - P(V|U)} : U \in \mathcal U, P(U)>0, V \in \mathcal V \}
$.
This notion gives rise to the sequence of $\phi$-mixing coefficients $\phi_m,~m \in \N $ of $\boldsymbol{\theta}$ where for each $m \in \N$ 
\begin{align*}
\phi_m&:=\sup_{j \in \N}\phi(\sigma(\{\theta_{t}: 1\leq t \leq j\}),\sigma(\{\theta_{t}: t \geq j+m\})).
\end{align*}
We further assume that the process $\boldsymbol{\theta}$ is $\phi$-mixing so that 
$\lim_{m\rightarrow \infty}\varphi_m=0$. 
Let $\A$ be the unit ball in $\R^d$, which we call the {\em action space}. 
The linear bandit problem considered in this paper can be formulated as the following repeated game. 
At every time-step $t \in \mathbb N$, the player chooses an {\em action} from $ \A$ according to a mapping $X_t:\Omega \rightarrow \A$ and receives as {\em pay-off} the inner product 
$\langle \theta_t, X_t\rangle$ 
between $\theta_t$ and $X_t$. 
The objective is to maximize the expected sum of the accumulated pay-offs. Let $\mathcal{F}_t,~{t \geq 0}$ be a filtration that tracks the pay-offs $\langle \theta_t,X_t \rangle$ obtained in the past $t$ rounds, i.e. $\mathcal{F}_0 = \{\emptyset,\Omega\}$, and $\mathcal{F}_{t} = \sigma(\{\langle \theta_1,X_1 \rangle, \dots, \langle \theta_t, X_t\rangle\})$ for $~t \geq 1$. Each mapping
$X_t,~t \geq 1$ is assumed to be measurable with respect to $\mathcal{F}_{t-1}$; this is equivalent to stating that $X_t$ for each $t \geq 1$ can be written as a function of the past pay-offs up to $t-1$. 
The sequence $\boldsymbol{\pi}:= (X_t,~t \in \N)$ is called a {\em policy}. 
Let $\Pi = \{ \boldsymbol{\pi}:= (X^{(\boldsymbol{\pi})}_t,~{t\geq 1}) : X^{(\boldsymbol{\pi})}_t \text{ is } \mathcal{F}_{t-1} \text{-measurable for all } t\geq 1\}$  denote the space of all possible policies and define
\begin{align}\label{eq:defn:nu}
\nu_n = \sup_{\boldsymbol{\pi} \in \Pi} \sum_{t=1}^n \E \langle \theta_t, X^{(\boldsymbol{\pi})}_t \rangle.
\end{align}
To simplify notation, we may sometimes write $X_t$ for $X^{(\boldsymbol{\pi})}_t$ when the policy $\boldsymbol{\pi}$ is clear from the context.

\section{\sc \bfseries Main Results}
Consider the restless linear bandit problem formulated in Section~\ref{sec:pre}.
Let $\theta^* = \E\;\theta_1 $ be the (stationary) mean of the process $(\theta_t,~ t \in \N)$. In light of the well-known computational hardness of optimal switching strategies for restless bandits \cite{PT99}, we aim to approximate $\nu_n$ via the following relaxation
\begin{equation}\label{eq:tildenu}
\tilde{\nu}_n = \sup_{\boldsymbol{\pi} \in \Pi}  ~\sum_{i=1}^n \E \langle \theta^*,X^{(\boldsymbol{\pi})}_t \rangle\;.
\end{equation}
A natural first objective is thus to quantify the error $\nu_n -\tilde{\nu}_n$ incurred by this approximation. We have the following result which can be considered as an analogue of \cite[Proposition~9]{GK19}. 
\begin{proposition}\label{prop:oraclediff}
Let $\varphi_1$ be the first $\varphi$-mixing coefficient of the process $(\theta_t,~ t \in \N)$.
For every $n\geq 1$ it holds that
\[
\nu_n - \tilde{\nu}_n \leq 2 n \phi_1 \norm{ \theta_t}_{\L_{\infty}}\;.
\]
\end{proposition}
\begin{proof} 
Consider an arbitrary policy $\boldsymbol{\pi}=(X_t^{(\boldsymbol{\pi})},~t \in \N)$ and any $t \in \N$; note that from this point on the notation $X_t$ will be used to denote $X_t^{(\boldsymbol{\pi})}$. Let $\tilde{P}_{t}:=P \circ X_t^{-1}$ be the pushforward measure on the action space $\A$ under $X_t$. 
Fix an $\epsilon >0$. As follows from Vitali's covering theorem \cite[Theorem 2.8]{M99} there exists a set of disjoint closed balls $\{B_i: i \in \N \}$ of radii at most $\epsilon$, that covers the action space $\A$ (i.e. the unit ball in $\R^d$),  
such that $\tilde{P}_{t}(\A \setminus \cup_{i \in \N} B_i)=0$. 
Note that by assumption $\theta_t \in \L_{\infty}(\Omega; \R^d)$ so that $\norm{\theta_t}_{\L_\infty}:=\sup_{\omega \in \Omega} \norm{\theta_t(\omega)}_2 < \infty$.
Consider a ball $B_i$ for some $i \in \N$ and denote its center by $\bar{x}_i \in \R^d$. Let $E_i:=\{X_t \in B_i\}$ denote the pre-image of $B_i$ under $X_t$.
Since $B_i$ is of radius at most $\epsilon$, it holds that,
\begin{align}\label{eq:centering_theta_t}
\E \langle \theta_t,X_t -\bar{x}_i\rangle \one_{B_i}(X_t) \leq \E \sup_{x \in B_i} \langle \theta_t, x -\bar{x}_i\rangle \one_{B_i}(X_t) \leq  \int_{E_i} \epsilon  \norm{\theta_t}_2 dP = \epsilon P(E_i) \norm{\theta_t}_{\L_\infty}.
\end{align}
Similarly, for $\theta^* = \E\; \theta_t$ we have,
\begin{align}\label{eq:centering_theta*}
\E\langle \theta^*,X_t -\bar{x}_i\rangle \one_{B_i}(X_t) \leq \epsilon P(E_i)\norm{\theta_t}_{\L_\infty}.
\end{align} 
Moreover, noting that $\A$ is the unit ball in $\R^d$, by Cauchy-Schwarz inequality, for each $x \in \A$ we have,
\begin{equation}\label{eq:norm_theta_t}
\norm{\langle \theta_t, x \rangle }_{\L_{\infty}} =
\sup_{\omega \in \Omega}|\langle \theta_t(\omega),x \rangle | \leq \sup_{\omega \in \Omega} \norm{\theta_t(\omega)}_2 
=\norm{\theta_t}_{\L_{\infty}} <\infty.
\end{equation}
Hence, as follows from  \cite[Theorem 4.4(c2) - vol. pp. 124]{BRA07} it holds that
\begin{align}
\frac{\norm{\E (   \langle \theta_t, \bar{x}_i\rangle  | \F_{t-1}) -  \E \langle \theta_t,\bar{x}_i \rangle}_{\L_\infty}}{\norm{\langle \theta_t, \bar{x}_i\rangle }_{\L_{\infty}}} 
&\leq \sup_{Y \in \mathcal \L_{\infty}(\Omega,\sigma(\theta_t),P )} \frac{\norm{\E(Y|\F_{t-1})-\E Y}_{\L_\infty}}{\norm{Y}_{\L_\infty}}\\
& = 2\phi(\F_{t-1},\sigma(\theta_t)) \\
&\leq 2\phi_1  \label{eq:BRA_4.4}
\end{align}
where $\mathcal \L_{\infty}(\Omega,\sigma(\theta_t),P )$ denotes the set of all $\sigma(\theta_t)$-measurable, real-valued random variables such that $\sup_{\omega \in \Omega} |Y(\omega)| <\infty$ almost surely.
We have,
\begin{align}
\left |\E ( (\langle \theta_t, \bar{x}_i\rangle - \langle \theta^*,\bar{x}_i \rangle) \one_{B_i}(X_t)) \right|
&=\left |\E \left( \E (\langle \theta_t, \bar{x}_i\rangle  \one_{B_i}(X_t) | \F_{t-1})\right) - \E \langle \theta^*,\bar{x}_i \rangle \one_{B_i}(X_t) \right|\\
&=\left |\E \left(\one_{B_i}(X_t)  \E (\langle \theta_t, \bar{x}_i\rangle  | \F_{t-1})\right) -  \E \langle \theta^*,\bar{x}_i \rangle \one_{B_i}(X_t) \right| \label{eq:meas}\\
&\leq \E\left(\one_{B_i}(X_t)\left |\E (\langle \theta_t, \bar{x}_i\rangle  | \F_{t-1}) -  \E \langle \theta_t,\bar{x}_i \rangle \right| \right) \\
&=\int_{E_i}\left |\E (\langle \theta_t, \bar{x}_i\rangle  | \F_{t-1}) -  \E \langle \theta_t,\bar{x}_i \rangle \right| dP\\
&\leq \int_{E_i}\norm{\E (\langle \theta_t, \bar{x}_i\rangle  | \F_{t-1}) -  \E \langle \theta_t,\bar{x}t_i \rangle}_{\L_\infty}  dP\\
&=P(E_i)\norm{\E (\langle \theta_t, \bar{x}_i\rangle  | \F_{t-1}) -  \E \langle \theta_t,\bar{x}_i \rangle}_{\L_\infty} \\
&\leq 2\phi_1 P(E_i) \norm{\theta_t}_{\L_{\infty}}\label{eq:expected_diff_phi1}
\end{align}
where \eqref{eq:meas} follows from noting that $X_t$ is $\F_{t-1}$-measurable, and \eqref{eq:expected_diff_phi1} follows from \eqref{eq:norm_theta_t} and \eqref{eq:BRA_4.4}.
Thus,
\begin{align}
&\left |\E\;\left ((\langle \theta_t,X_t \rangle- \langle \theta^*, X_t\rangle)  \one_{B_i}(X_t) \right)\right|\\
&\leq \left |\E\;\langle \theta_t,X_t -\bar{x}_i\rangle \one_{B_i}(X_t) \right | + 
\left |\E\;\langle \theta^*,X_t -\bar{x}_i\rangle \one_{B_i}(X_t) \right |
+\left |\E  (\langle \theta_t, \bar{x}_i\rangle - \langle \theta^*,\bar{x}_i \rangle) \one_{B_i}(X_t) \right|\\
&\leq 2\epsilon P(E_i) \norm{\theta_t}_{\L_{\infty}}+2\phi_1 P(E_i)\norm{ \theta_t}_{\L_{\infty}}\label{eq:expected_diff_Bi}
\end{align}
where \eqref{eq:expected_diff_Bi} follows from \eqref{eq:centering_theta_t}, \eqref{eq:centering_theta*} and  \eqref{eq:expected_diff_phi1}. By \eqref{eq:expected_diff_Bi}, noting that the events $E_i,~i \in \N$ partition $\Omega$,
and by applying the dominated convergence theorem, we have,
\begin{align}
\left |\E\;(\langle \theta_t,X_t \rangle- \langle \theta^*, X_t\rangle)\right| 
&\leq 
\sum_{i\in \N} \left |\E\;((\langle \theta_t,X_t \rangle- \langle \theta^*, X_t\rangle)  \one_{B_i}(X_t))\right|\\
& \leq \sum_{i \in \N} 2P(E_i)(\epsilon+\phi_1) \norm{ \theta_t}_{\L_{\infty}}
=2(\epsilon+\phi_1) \norm{ \theta_t}_{\L_{\infty}}.
\end{align}
Finally, since this holds for every $\epsilon$, we obtain,
\begin{align}
\nu_n - \tilde{\nu}_n \leq  \sup_{\boldsymbol{\pi} \in \Pi}
\sum_{t=1}^n |\E (\langle \theta_t, X_t \rangle - \langle \theta^*, X_t \rangle)| 
\leq 2n \phi_1 \norm{ \theta_t}_{\L_{\infty}}
\end{align}
and the result follows. 
\end{proof}
Denote by $\mathcal{R}_{\boldsymbol{\pi}}(n)$ the regret with respect to \eqref{eq:tildenu} incurred by a policy $\boldsymbol{\pi}=(X_t^{(\boldsymbol{\pi})},~t \in \N)$ after $n$ rounds, i.e. 
\begin{equation}\label{eq:defn:regret}
\mathcal{R}_{\boldsymbol{\pi}}(n) := \tilde{\nu}_n - \sum_{t=1}^n \E \langle \theta_t, X^{(\boldsymbol{\pi})}_t \rangle.
\end{equation}
 We propose LinMix-UCB, outlined in Algorithm~\ref{alg:fh} below, which aims to minimize the regret in the case where the process $(\theta_t,~ t \in \N)$ has an exponential mixing rate, so that there exists some $a, \gamma \in (0,\infty)$ such that  $\varphi_m \leq a e^{-\gamma m}$ for all $m \in \N$.  
\begin{algorithm}
\caption{LinMix-UCB (finite horizon)}\label{alg:fh}
\begin{algorithmic}
\Require horizon $n$; regularization parameter $\lambda$; $\varphi$-mixing rate parameters: $a,\gamma \in (0,\infty)$ 
\State~
\State{Specify block-length $k$ given by \eqref{eq:kn}}
\State~
\For{$m=0,1,2,\dots,\lfloor n/k \rfloor$}
    \For{$\ell=1,2,\dots,k$}
        \State $t \gets mk+\ell$
        \If {$m=0$}
            \State $X_t \gets \x_0$ \Comment{{\em $\x_0$ is a fixed unit vector in $\A$}}
        \Else
            \State $X_t \gets \argmax_{x \in \A} \max_{\theta \in C_{\max\{0,m-1\}}} \langle x,\theta \rangle $
        \EndIf
        \State Play action $X_t$ to obtain reward $Y_t$
        \If {$\ell=1$}
            \State Calculate confidence ellipsoid $C_m$ given by \eqref{eq:cm}  \Comment{{\em to be used at least $k$ steps later}}
        \EndIf
	\EndFor
\EndFor
\end{algorithmic}
\end{algorithm}

 The proposed algorithm is inspired by such UCB-type approaches as LinUCB and its variants, including LinRel \cite{A02} and OFUL
 \cite{ADC11}, all of which are 
 designed for linear bandits in the simpler iid noise setting, see \cite[Chapter 19]{LS20} and references therein. The main challenge in our setting is to devise an exploration-exploitation strategy that is robust against the long-range dependencies in the process $(\theta_t,~t \in \N)$. We construct confidence ellipsoids around the empirical estimates of $\theta^*$ obtained via ``near-independent'' samples, and similarly to LinUCB and its variants, we rely on the principle of Optimism in the Face of Uncertainty. 
More specifically, the algorithm works as follows. 
 A finite horizon $n$ is divided into intervals of length 
\begin{equation}\label{eq:kn}
k=\max\left \{ 1, \left \lceil \frac{1}{\gamma}\log\left( \frac{6 a \gamma n^{2} }{1+ 4\sqrt{n} \norm{\theta_t}_{\L_{\infty}} + \sqrt{\frac{8 d\times n \log(n(1+\frac{n}{\lambda d}))}{\lambda}}} \right) \right \rceil \right \}
\end{equation}
where $\lambda$ is a regularization parameter used in the estimation step. 
At every time-step $t=mk+1,~m=0,1,\dots,\lfloor n/k \rfloor$ which marks the beginning of each interval of length $k$, the pay-offs $Y_s:=\langle \theta_s,X_s \rangle$ collected every $k$-steps at $s=m'k+1,~m'=0,1,\dots,\max\{0,m-1\} $, are used to generate a regularized least-squares estimator $\wh{\theta}_m$ of $\theta^*=\E\theta_t$, and in turn, produce a confidence ellipsoid $C_m$. That is, for each $m=0,1,\dots,\lfloor n/k \rfloor$ we have
\begin{align}\label{eq:thetahat}
\wh{\theta}_m:=\argmin_{\theta \in \Theta}\left ( \sum_{\substack{m'=0\\ s=m'k+1}}^{\max\{0,m-1\}}(Y_{s}-\langle \theta, X_s\rangle )^2 + \lambda \norm{\theta}_2^2\right)
\end{align}
where the regularisation parameter $\lambda >0$ ensures a unique solution given by 
\begin{align}
    \wh{\theta}_m=(\lambda I + V_m)^{-1}\sum_{\substack{m'=0\\ s=m'k+1}}^{\max\{0,m-1\}} Y_s X_s 
\end{align}
with $\displaystyle{V_m :=  \sum_{\substack{m'=0\\ s=m'k+1}}^{\max\{0,m-1\}} X_s X_s^{\top}}$ takes values in $\R^{d\times d}$ and $I$ is the identity matrix in $\R^{d \times d}$. This gives rise to the following confidence ellipsoid
\begin{align}\label{eq:cm}
C_m:= \left\{ \theta \in \Theta: \norm{\theta-\wh{\theta}_{m}}_{\zeta^2(\lambda I + V_{m})}^2 \leq b_n\right \}
\end{align}
where $\norm{x}_{A}^2:= x^{\top} A x$ for $x \in \R^d$ and $A \in \R^{d \times d}$ 
and $b_n$ is chosen such that 
\begin{align}\label{eq:bnb}
\sqrt{b_n} 
&=2\sqrt{\lambda}\norm{\theta_t}_{\L_{\infty}}+\sqrt{2\log n+ d\log\left(1+\frac{n}{k_n \lambda d}\right )}.
\end{align}
We are now in a position to analyze the regret of the proposed algorithm. 
\begin{theorem}\label{thm:regfh}
Suppose that the stationary $\varphi$-mixing process $(\theta_t,~ t \in \N)$ has an exponential mixing rate, so that there exists some $a, \gamma \in (0,\infty)$ such that  $\varphi_m \leq a e^{-\gamma m}$ for all $m \in \N$. 
The regret (with respect to $\tilde{\nu}_n)$ of LinMix-UCB (finite horizon) played for $n \geq \left \lceil \frac{3 a \gamma \sqrt{\lambda}}{2 \sqrt{\lambda} \norm{\theta_t}_{\L_{\infty}}+\sqrt{2}} \right \rceil $ rounds satisfies
\begin{align}
\frac{\mathcal R_{{\boldsymbol{\pi}}}(n) }{\norm{\theta_t}_{\L_{\infty}}}
&\leq \frac{1}{n}+C \log (n) \sqrt{2d n\log(n(1+\frac{n}{\lambda d}))} 
\end{align}
where 
$C:= \left (\frac{12(\sqrt{2\lambda}+4\sqrt{2\lambda}\norm{\theta_t}_{\L_{\infty}}+1)}{\gamma\sqrt{2\lambda} }\right) $
$\lambda >0$ is the regularization parameter. 
\end{theorem}
\begin{proof}
For a fixed $k \in \N$, consider the sub-sequence $\theta_{ik+1},~i=0,1,2,\dots$ of the stationary sequence of $\R^d$-valued random variables $\theta_t,~t \in \N$. 
Let $U_j,~j \in \N$ be a sequence of iid random variables uniformly distributed over $[0,1]$ such that each $U_j$ is independent of $\sigma(\{\theta_t: t \in {\N}\})$.
Set $\tilde{\theta}_0=\theta_1$. 
As follows from Berbee's coupling lemma \cite{BER79} - see also \cite[Lemma 5.1, pp. 89]{RIO17} -  there exists a random variable $\tilde{\theta}_1 = g_1(\tilde{\theta}_0,\theta_{k+1},U_1)$ where $g_1$ is a measurable function from $\R^d \times \R^d \times [0,1]$ to $\R^d$
such that $\tilde{\theta}_1$ is independent of $\tilde{\theta}_0$, has the same distribution as $\theta_{k+1}$ and 
\[\Pr(\tilde{\theta}_1 \neq \theta_{k+1}) = \beta(\sigma(\tilde{\theta}_0),\sigma(\theta_{k+1})).\]
Similarly, there exists a random variable 
$\tilde{\theta}_2 = g_2((\tilde{\theta}_0,\tilde{\theta}_1),\theta_{2k+1},U_2)$ where $g_2$ is a measurable function from $(\R^{d})^2 \times \R^{d} \times [0,1]$ to $ \R^{d}$ such that $\tilde{\theta}_2$ is independent of $(\tilde{\theta}_0,\tilde{\theta}_1)$, has the same distribution as $\theta_{2k+1}$ and 
$
    \Pr(\tilde{\theta}_2 \neq \theta_{2k+1}) = \beta(\sigma(\tilde{\theta}_0,\tilde{\theta}_1),\sigma(\theta_{2k+1}))
$.
Continuing inductively in this way, at each step $j=3,4,\dots$, by Berbee's coupling lemma \cite{BER79}, there exists a random variable 
$\tilde{\theta}_j = g_j((\tilde{\theta}_0,\tilde{\theta}_1,\dots,\tilde{\theta}_{j-1}),\theta_j,U_j)$ where $g_j$ is a measurable function from $(\R^{d})^j \times \R^{d} \times [0,1]$ to $ \R^{d}$ such that $\tilde{\theta}_j$ is independent of $(\tilde{\theta}_0,\tilde{\theta}_1,\dots,\tilde{\theta}_{j-1})$, has the same distribution as $\theta_j$ and that 

\begin{equation}\label{eq:berb1}
    \Pr(\tilde{\theta}_j \neq \theta_j) = \beta(\sigma(\tilde{\theta}_0,\tilde{\theta}_1,\dots,\tilde{\theta}_{j-1}),\sigma(\theta_j)).
\end{equation}
Following a standard argument (see, e.g.~\cite[Lemma~6]{GK24}) it can be shown that,
\begin{equation}\label{eq:berb2}
\beta(\sigma(\tilde{\theta}_0,\tilde{\theta}_1,\dots,\tilde{\theta}_{j-1}),\sigma(\theta_j))\leq \beta_k
\end{equation}
for all $j \in \N $.
The sequence of random variables $\tilde{\theta}_j,~j=0,1,2,\dots$ can be used to construct a ``ghost'' pay-off process $(\bar{\theta}_t,~t \in \N)$ as follows. Let $\bar{\theta}_{ik+1} = \tilde{\theta}_{i}$ for all $i =0,1,2,\dots$ and take $\bar{\theta}_t$ to be an independent copy of $\theta_1$ for all $t =ik+2,\dots,(i+1)k,~i=0,1,2,\dots$.  
Denote by $\boldsymbol{\pi}:=(X_t,~t \in \N)$ the policy induced by Algorithm~\ref{alg:fh} when the process $(\theta_t,~t \in \N)$ is used to generate the pay-offs $Y_t := \langle X_t,\theta_t\rangle$. Similarly let $\boldsymbol{\bar{\pi}}:=(\bar{X}_t,~t \in \N)$ be the policy generated by Algorithm~\ref{alg:fh} when the sequence $(\bar{\theta}_t,~t \in \N)$ produces the pay-offs $\bar{Y}_t := \langle \bar{X}_t,\bar{\theta}_t\rangle$ at each $t \in \N$. 
For a fixed $n \in \N$, define the event 
\begin{equation}\label{eq:En}
E_n :=\{\exists i \in 0,1,\dots,(\lfloor n/k \rfloor)-1: \theta_{ik+1} \neq \bar{\theta}_{ik+1}\}
\end{equation}
and observe that as follows from the above coupling argument, i.e. by \eqref{eq:berb1} and \eqref{eq:berb2}, it holds that
\begin{equation}\label{eq:Enmax}
\Pr(E_n) \leq n\beta_k/k.
\end{equation}
Let $\G_0 = \bar{\G}_0 := \{\emptyset, \Omega\}$.
Define the filtrations $\mathcal G_{m}:=\sigma(\{\theta_{ik+1}:i=0,1,\dots,m-1\})$ and 
$\bar{\G}_m:=\sigma(\{\bar{\theta}_{ik+1}: i=0,1,\dots,m-1\})$ for $m=1,2,\dots,\lfloor n/k \rfloor$. By design, for $t=1,\dots,k$, the action $X_t$ is set to a pre-specified constant $\x_0 \in \A$ (independent of the data), and is thus simply $\G_0$-measurable. Observe that the first confidence ellipsoid $C_0$ which is generated at $t=1$ is not used directly, but only after $k$ steps.
For each time-step $t=mk+\ell$ with $m \in 1,\dots,\lfloor n/k \rfloor$ and $\ell=1,2,\dots,k$, the action $X_t$ depends on the confidence ellipsoid $C_{\max\{0,m-1\}}$, which is in turn updated at least $k$ time-steps prior to $t$. More specifically, $X_t$ is measurable with respect to $\G_{m}$. As a result, there exists a measurable function $f_t: \underset{\text{$m$ times}}{\underbrace{ \R^d \times \dots \times \R^d}}  \to \A$ such that $X_t=f_t(\theta_1,\theta_{k+1},\theta_{2k+1},\dots,\theta_{(m-1) k + 1})$; in words, $f_t$ is a mathematical representation of Algorithm~\ref{alg:fh} at time $t$. Similarly, noting that the same algorithm is applied to $(\bar{\theta}_t,,~t \in \N)$, it holds that $\bar{X}_t={f}_t(\bar{\theta}_1,\bar{\theta}_{k+1},\bar{\theta}_{2k+1},\dots,\bar{\theta}_{(m-1) k + 1})$. As a result, 
for each $i =1,\dots,\lfloor n/k \rfloor$, we have,
\begin{align}
Y_{ik+1}\one_{E_n^c} &= \langle \theta_{ik+1}, X_{ik+1} \rangle \one_{E_n^c} \\
&= \langle \theta_{ik+1}, f_{ik+1}(\theta_1,\dots, \theta_{(i-1)k+1}) \rangle \one_{E_n^c} \\
&= 
\langle \bar{\theta}_{ik+1}, {f}_{ik+1}(\bar{\theta}_1,\dots, \bar{\theta}_{(i-1)k+1}) \rangle \one_{E_n^c}\\
&= \bar{Y}_{ik+1}\one_{E_n^c}, \label{eq:YandYbarequal}
\end{align}
almost surely. 
On the other hand, by a simple application of Cauchy-Schwarz and Hölder inequalities, it is straightforward to verify that for each $t = 1,\dots, n$ we have,
\begin{align}\label{eq:expectedY_holder1}
    \E |Y_{t} \one_{E_n}| 
    = \int_{E_n} |\langle \theta_{t},X_{t}\rangle|  dP  \leq  P(E_n)\norm{\theta_t}_{\L_{\infty}} 
\end{align}
and similarly,
\begin{align}\label{eq:expectedY_holder2}
\E |\bar{Y}_{t} \one_{E_n}| \leq P(E_n)\norm{{\theta}_t}_{\L_{\infty}}.
\end{align}
It follows that
\begin{align}
\sum_{i=0}^{\lfloor n/k \rfloor}\left |\E Y_{ik+1} - \E \bar{Y}_{ik+1}\right | 
    &=\sum_{i=0}^{\lfloor n/k \rfloor}\left |\E (Y_{ik+1} \one_{E_n}) - \E({Y}_{ik+1} \one_{E_n}) \right|\label{eq:enc}\\
    &\leq 2 n \norm{\theta_t}_{\L_{\infty}} P(E_n)/k\label{eq:cauchy}\\
    &\leq 2n^2\norm{\theta_t}_{\L_{\infty}} \beta_k/k^2\label{eq:regb1}
\end{align}
where \eqref{eq:enc} follows from \eqref{eq:YandYbarequal},  \eqref{eq:cauchy} follows from \eqref{eq:expectedY_holder1} and \eqref{eq:expectedY_holder2}, and \eqref{eq:regb1} is due to \eqref{eq:Enmax}.
Next, let us consider the time-steps within each segment. 
Fix any $t= m k+ \ell$ for some $m \in 1,\dots,\lfloor n/k \rfloor$ and some $\ell \in 2,\dots,k$. 
It is straightforward to verify that \cite[Theorem 4.4(c2) - vol.1 pp. 124]{BRA07} can be extended to the case of vector-valued random variables, by an analogous argument based on simple functions where absolute values of constants are replaced by norms. 
This leads to,
\begin{align}\label{eq:BRAVEC}
\norm{\E(\theta_t|\bar{\G}_{m} \vee \G_{m}) - \E \theta_t}_{\L_{\infty}} \leq 2\phi(\bar{\G}_{m} \vee\G_{m},\sigma(\theta_t))\norm{\theta_t}_{\L_{\infty}} \leq 
2\phi_{k}\norm{\theta_t}_{\L_{\infty}}.
\end{align}
Define the event $U_{m}:=\{\exists i \in 0,1,\dots,m-1: \theta_{ik+1} \neq \bar{\theta}_{ik+1}\}$. 
Observe that as with \eqref{eq:En} 
we have
\begin{equation}\label{eq:PUm}
\Pr(U_m) \leq m\beta_k
\end{equation}
so that similarly to \eqref{eq:expectedY_holder1}, it holds that,
\begin{align}\label{eq:expectedY_holder_um}
    \max\{\E |Y_{t} \one_{U_m}|, 
    \E |\bar{Y}_{t} \one_{U_m}|\} \leq m\beta_k\norm{{\theta}_t}_{\L_{\infty}}.
\end{align}
for $t= m k+ \ell$ with some $m \in 1,\dots,\lfloor n/k \rfloor$ and $\ell \in 2,\dots,k$ fixed above. 
Moreover, 
\begin{align}
\E (\langle \bar{\theta}_t,\bar{X}_t \rangle \one_{U_m^c}) 
&=\E(\one_{U_m^c}\E (\langle \bar{\theta}_t,X_t\rangle|\bar{\G}_{m } \vee \G_{m})) \label{eq:Umc1}\\
&=\E(\one_{U_m^c}\langle X_t, \E(\bar{\theta}_t|\bar{\G}_{m } \vee \G_{m} )\rangle) \label{eq:Umc2}\\
&=\E (\one_{U_m^c} \langle X_t, \E\;\theta_t \rangle)\label{eq:Umc3}
\end{align}
where \eqref{eq:Umc1} follows from the definition of $U_m^c$ and the fact that $\one_{U_m^c}$ is measurable with respect to  $\bar{\G}_{m } \vee \G_{m}$, and \eqref{eq:Umc2} follows from noting that $X_t$ is measurable with respect $\G_t$. Finally, \eqref{eq:Umc3} is due to the stationarity of $\theta_t$ together with the fact that by construction $\bar{\theta}_t$ is an independent copy of $\theta_1$ for this choice of $t$ (within the segments).
Similarly, we have,
\begin{align}
\E (\langle {\theta}_t,{X}_t \rangle \one_{U_m^c}) 
&=\E(\one_{U_m^c}\E (\langle {\theta}_t,{X}_t \rangle |\bar{\G}_{m} \vee \G_{m})) \label{eq:Umc4}\\
&=\E(\one_{U_m^c}\langle X_t,\E(\theta_t|\bar{\G}_{m} \vee \G_{m}) \rangle).\label{eq:Umc5}
\end{align}
Hence, for any $m \in 1,\dots,\lfloor n/k \rfloor$ and  $t= m k+ \ell$ for some $\ell \in 2,\dots,k$ we have,
\begin{align}
|\E ((\langle \theta_t, X_t \rangle - \E \langle \bar{\theta}_t,\bar{X}_t\rangle)\one_{U_m^c})|
&\leq \E (\one_{U_m^c} \langle X_t, |\E(\theta_t|\bar{\G}_{m} \vee \G_{m}) - \E\;\theta_t| \rangle))\label{eq:Umcdiff1}\\
&=\int_{U_m^c} \langle X_t, |\E(\theta_t|\bar{\G}_{m} \vee \G_{m}) - \E\;\theta_t| \rangle dP \label{eq:Umcdiff2}\\
&\leq \int_{U_m^c }\norm{X_t}_2 \norm{\E(\theta_t|\bar{\G}_{m} \vee \G_{m}) - \E\;\theta_t}_2 dP \label{eq:Umcdiff3}\\
&\leq P(U_m^c)\norm{ \E(\theta_t|\bar{\G}_{m} \vee \G_{m}) - \E\;\theta_t}_{\L_{\infty}}\label{eq:Umcdiff4}\\
&\leq 2 \phi_{k} \norm{\theta_t}_{\L_{\infty}}\label{eq:Umcdiff5}
\end{align}
where \eqref{eq:Umcdiff1} follows from \eqref{eq:Umc3} and \eqref{eq:Umc5}; \eqref{eq:Umcdiff3} and \eqref{eq:Umcdiff4} follow from Cauchy-Schwarz and Hölder inequalities respectively, and \eqref{eq:Umcdiff5} follows from \eqref{eq:BRAVEC}. 
We obtain,
\begin{align}
\sum_{m=1}^{\lfloor n/k \rfloor}\sum_{\ell=2}^k&|\E (Y_{mk+\ell} - \bar{Y}_{mk+\ell})| \nonumber \\
&=
\sum_{m=0}^{\lfloor n/k \rfloor}\sum_{\ell=2}^k|\E ((Y_{mk+\ell} - \bar{Y}_{mk+\ell})\one_{U_m^c})| + |\E ((Y_{mk+\ell} - \bar{Y}_{mk+\ell})\one_{U_m})|\label{eq:mkdiff0}\\
&\leq  2 n \phi_{k} \norm{\theta_t}_{\L_{\infty}} + 2 n^2 \beta_k \norm{\theta_t}_{\L_{\infty}}
\label{eq:mkdiff1}\\
&\leq 4n^2\phi_k \norm{\theta_t}_{\L_{\infty}}\label{eq:mkdiff2}
\end{align}
where \eqref{eq:mkdiff0} follows from noting that by design, $\E\bar{Y}_t = \E Y_t$ for all $t \in 1,\dots,k$ as the algorithm sets $X_t = \x_0$ for some constant $\x_0 \in \A$ independent of the data, \eqref{eq:mkdiff1} follows from \eqref{eq:expectedY_holder_um} and \eqref{eq:Umcdiff5},  and \eqref{eq:mkdiff2} is due to the fact that in general $\beta_k \leq \phi_k$ for all $k \in \N$. Therefore, by \eqref{eq:regb1} and \eqref{eq:mkdiff2} we obtain
\begin{align}
\left |\mathcal{R}_{\boldsymbol{\pi}(n)}-\mathcal{R}_{\bar{\boldsymbol{\pi}}}(n)\right|
    &\leq \sum_{i=0}^{\lfloor n/k \rfloor}\left |\E Y_{ik+1} - \E \bar{Y}_{ik+1}\right |  + \sum_{m=0}^{\lfloor n/k \rfloor}\sum_{\ell=2}^{k}  \left |\E Y_{mk+\ell} - \E \bar{Y}_{mk+\ell}\right | \nonumber \\
    &\leq 6n^2\phi_k \norm{\theta_t}_{\L_{\infty}}.\label{eq:linkingregrets}
\end{align}
It remains to calculate the regret of $\bar{\boldsymbol{\pi}}=\{\bar{X}_t:t\in 1,\dots,n\}$. 
The pay-off $\bar{Y}_t= \langle \bar{\theta}_t, \bar{X}_t \rangle$ obtained via the policy $\bar{\boldsymbol{\pi}}$ at time-step $t$ can be decomposed as 
\begin{align}
&\bar{Y}_t =\langle \theta^*, \bar{X}_t \rangle +\eta_t&\qquad&\text{where}& \eta_t:=\langle \bar{\theta}_t-\theta^*, \bar{X}_t \rangle.
\end{align} 
For each $m =0,1,\dots, \lfloor n/k \rfloor$, set $S_m = \sum_{i=0}^m \eta_{ik+1}\bar{X}_{ik+1}$, define $\displaystyle{V_m :=  \sum_{\substack{m'=0\\ s=m'k+1}}^{m} \bar{X}_s \bar{X}_s^{\top}}$ taking values in $\R^{d\times d}$, and let $I \in \R^{d\times d}$ be the identity matrix.  
Consider the estimator $\doublehat{\theta}_m:=(\lambda I + V_m)^{-1}\sum_{\substack{m'=0\\ s=m'k+1}}^{m} \bar{Y}_s \bar{X}_s$, and observe that 
\begin{align}
    \doublehat{\theta}_m&=(\lambda I + V_m)^{-1}\left (S_m+\sum_{\substack{m'=0\\ s=m'k+1}}^{m}  \bar{X}_s \bar{X}_s^{\top} \theta^*\right)\\
    &=(\lambda I + V_m)^{-1}\left (S_m+V_m \theta^*\right).
\end{align}
Let $\zeta:=2\norm{\theta_t}_{\L_{\infty}}$. We can write,
\begin{align}
    \norm{\doublehat{\theta}_m - \theta^*}_{\zeta^2(\lambda I + V_m)} 
    &=\norm{(\lambda I + V_m)^{-1}\left (S_m+V_m \theta^*\right) - \theta^*}_{\zeta^2(\lambda I + V_m)}\\
    &= \norm{(\lambda I + V_m)^{-1} S_m+ ((\lambda I + V_m)^{-1}V_m -I)\theta^*}_{\zeta^2(\lambda I + V_m)} \\
    &\leq \norm{S_m}_{\zeta^2(\lambda I + V_m)^{-1}}+\zeta\lambda^{1/2}({\theta^*}^{\top}(I - (\lambda I + V_m)^{-1}V_m)\theta^*)^{1/2}\\
    &=\norm{S_m}_{\zeta^2(\lambda I + V_m)^{-1}}+\zeta\lambda({\theta^*}^{\top}(\lambda I + V_m)^{-1}\theta^*)^{1/2}\\
     &\leq \norm{S_m}_{\zeta^2(\lambda I + V_m)^{-1}}+ \zeta \lambda^{1/2}\norm{\theta^*}_2 \label{eq:Sm_err}
\end{align}
where \eqref{eq:Sm_err} follows from 
noting that $V_m$ is positive semi-definite and Löwner matrix order is reversed through inversion, so that $
{\theta^*}^{\top}(\lambda I + V_m)^{-1} \theta^* \leq {\theta^*}^{\top} (\lambda I )^{-1} \theta^* = \lambda^{-1}\norm{\theta^*}^2_2.
$

Observe that $\bar{X}_t$ for $t=mk+1,~m=0,1,\dots, \lfloor n/k \rfloor $ is $\bar{\G}_{m}$-measurable, and that by construction $\bar{\theta}_{ik+1}$ for $i=0,1,\dots,m$ are iid. Thus,  
$\E(\eta_t) 
= \E(\langle \bar{X}_t,\E(\theta^*-\theta_t|\bar{\G}_{m})\rangle) = 0
$. Furthermore, it is straightforward to verify that by Cauchy-Schwarz inequality and noting that $X_t$ takes values in the unit ball, $\eta_t$ is $\zeta$-subGaussian, i.e. 
for all $\alpha \in \R$ and every $m =0,1,\dots,\lfloor n/k \rfloor $ and $t=mk+1$ we have,
$\E(e^{\alpha \eta_t}|\bar{\G}_{m})) \leq e^{\alpha^2 \zeta^2/2}$
almost surely. In particular,
\begin{align}
\E(\exp\{\langle x,X_t \rangle \eta_t\}|\bar{\G}_{m})) \leq \exp\{\langle x,X_t \rangle^2 \zeta^2/2\} = \exp\left \{\frac{\zeta^2\norm{x}^2_{X_tX_t^{\top}}}{2}\right \},~\text{a.s.}\label{eq:zetabound}
\end{align}
for all $x\in \R^d$. 
Define $M_m(x):=\exp\{\langle x, S_m\rangle - \frac{ \zeta^2\norm{x}^2_{V_m}}{2} \}$ for $x \in \R^d$ and $m =0,1,\dots, \lfloor n/k \rfloor $. 
We have
\begin{align}
\E(M_{m}(x)|\bar{\G}_m) 
&=M_{m-1}(x)\E\left (\exp\left \{\eta_{mk+1}\langle x, \bar{X}_{mk+1}\rangle - \frac{\zeta^2}{2}\norm{x}^2_{\bar{X}_{mk+1}\bar{X}_{mk+1}^{\top}} \right\}\Big|\bar{\G}_m\right)\\
&=M_{m-1}(x) \exp\left \{- \frac{\zeta^2}{2}\norm{x}^2_{\bar{X}_{mk+1}\bar{X}_{mk+1}^{\top}} \right\} \E\left (\exp\left \{\eta_{mk+1}\langle x, \bar{X}_{mk+1}\rangle \right\}\Big|\bar{\G}_m\right)\label{eq:Mmismartin0}\\
&\leq M_{m-1}(x)\label{eq:Mmismartin}
\end{align}
where \eqref{eq:Mmismartin0} follows from the fact that $\bar{X}_{mk+1}$ is $\bar{\G}_m$-measurable, and \eqref{eq:Mmismartin} follows from \eqref{eq:zetabound}. 
Moreover, by \eqref{eq:zetabound}  and $\bar{\G}_0$-measurability of $\bar{X}_1$, for every $x \in \R^d$ it holds that
\begin{align}
\E(M_0(x)) 
&= \E\left ( -\frac{\zeta^2\norm{x}^2_{ \bar{X}_1 \bar{X}_1^{\top}}}{2} \E \left(\exp\left \{\langle x, \bar{X}_1\rangle \eta_1 \right \}\Big |\bar{\G}_0\right )\right )\leq 1. \label{eq:boundM0}
\end{align}

Let $W:\Omega \to \R^d$ be a $d$-dimensional Gaussian random vector with mean $\boldsymbol{0} \in \R^d$ and covariance matrix $(\zeta^{2}\lambda)^{-1} I \in \R^{d\times d}$; denote by $P_{W}$ its distribution on $\R^d$.  Define 
\begin{align}
\tilde{M}_m:=\int_{\R^d} M_m(x) d P_W(x)
\end{align}
for each $m \in 0,1,\dots,\lfloor n/k \rfloor $. Observe that by \eqref{eq:boundM0} and Fubini's theorem we have
\begin{align}\label{eq:exptildeM0bound}
\E \tilde{M}_0  =\E \left ( \int_{\R^d} M_0(x) dP_{W}(x)\right) = \int_{\R^d} \E M_0(x)dP_{W}(x) \leq 1
\end{align}
Furthermore, by completing the square in the integrand, we can write 
\begin{align}
\tilde{M}_m
&=\frac{1}{\sqrt{(2\pi)^d \det((\zeta^2\lambda)^{-1} I)}}\int_{\R^d} \exp\left \{\langle x, S_m\rangle - \frac{\zeta^2}{2} \norm{x}^2_{V_m} -\frac{1}{2}\norm{x}^2_{(\zeta^2\lambda I) }\right \}dx\\
&=\exp\left \{\frac{1}{2} \norm{S_m}^2_{\zeta^2(\lambda I +V_m)^{-1}}\right \}\left ( \frac{(\zeta^2\lambda)^d}{\det(\zeta^2(\lambda I +V_m))}\right)^{1/2}\\
&=\exp\left \{\frac{1}{2} \norm{S_m}^2_{\zeta^2(\lambda I +V_m)^{-1}}\right \}\left ( \frac{\lambda^d}{\det(\lambda I +V_m)}\right)^{1/2}
\label{eq:Mmcalc}
\end{align}
On the other hand, by Fubini's theorem together with \eqref{eq:Mmismartin}, we have that $\tilde{M}_m$ is a non-negative super-martingale, i.e.
\begin{align}
\E(\tilde{M}_m | \bar{\G}_m) 
=\int_{\R^d} \E(M_m(x) |\bar{\G}_m)dP_W \leq \int_{\R^d} M_{m-1}(x) = \tilde{M}_{m-1}.
\end{align}
As a result, by Doob's maximal inequality for every $\delta>0$ it holds that
\begin{align}\label{eq:doobidoo}
\Pr\left (\sup_{m \in \N} \log \tilde{M}_m \geq \log (1/\delta)\right )= \Pr\left (\sup_{m \in \N} \tilde{M}_m \geq \frac{1}{\delta}\right ) \leq \delta \E\tilde{M}_0 \leq \delta 
\end{align}
where the last inequality follows from \eqref{eq:exptildeM0bound}. 
By \eqref{eq:Mmcalc} and \eqref{eq:doobidoo}, we have
\begin{align}\label{eq:boundonSm}
\Pr\left (\left \{\exists \; m: \norm{S_m}^2_{\zeta^2(\lambda I + V_m)^{-1}} \geq 2\log \left(\frac{1}{\delta}\right) + \log\left(\frac{\det(\lambda I + V_m)}{\lambda^d}\right)\right \}\right )\leq \delta. 
\end{align}
Define $b_n^{\delta}$ such that  $\sqrt{b_n^{\delta}}:=2\sqrt{\lambda}\norm{\theta_t}_{\L_{\infty}}+\sqrt{2\log \left(\frac{1}{\delta}\right)+ d\log\left(1+\frac{n}{k \lambda d}\right )}$ and let
\begin{align}
C_m:= \left\{ \theta \in \Theta: \norm{\theta-\wh{\theta}_{m}}_{\zeta^2(\lambda I + V_{m})}^2 \leq b_n^{\delta}
\right \}.
\end{align}
By \eqref{eq:Sm_err} and \eqref{eq:boundonSm}, with probability at least $1-\delta$ it holds that,
\begin{align}
\norm{\doublehat{\theta}_m - \theta^*}_{\zeta^2(\lambda I + V_m)} 
\leq \zeta \sqrt{\lambda}+\sqrt{2\log \left(\frac{1}{\delta}\right) + \log\left(\frac{\det(\lambda I + V_m)}{\lambda^d}\right )}
\leq b_n^{\delta}
\end{align}
where the second inequality follows from the definition of $\zeta$ as well as from \cite[Equation 20.9, pp. 261]{LS20}. 
Then, it immediately follows that 
\begin{align}\label{eq:thetastarinCm}
\Pr(\{\exists m: \theta^* \notin C_m\}) \leq \delta.
\end{align}
Consider the instantaneous regret 
$r_{mk+1}:=\langle \theta^*,\bar{X}_{mk+1}^*-\bar{X}_{mk+1} \rangle $ of  $\bar{\boldsymbol{\pi}}$ 
for $m=1,\dots,\lfloor n/k \rfloor -1$, where $\bar{X}_m^*$ is an optimal action at $mk+1$ 
so that $\bar{X}_m^* \in \argmax_{x \in C_{m-1}} \langle \theta^*,x\rangle $ almost surely. Let us recall that the algorithm selects $\bar{X}_{mk+1} \in \argmax_{x \in \A } \max_{\theta \in C_{\max\{0,m-1\}}}\langle \theta, x\rangle $.
With probability at least $1-\delta$ we have,
\begin{align}
\langle \theta^*,\bar{X}_{mk+1}^* \rangle 
&\leq   \max_{\theta \in C_{\max\{0,m-1\}}}  \langle \theta, \bar{X}^*_{mk+1}\rangle \\
&\leq \max_{\theta \in C_{\max\{0,m-1\}}} \langle \theta, \bar{X}_{mk+1}\rangle.  
\end{align}
Fix some $\bar{\theta}_{mk+1} \in \argmax_{\theta \in C_{\max\{0,m-1\}}} \langle \theta,\bar{X}_{mk+1} \rangle$. With probability at least $1-\delta$ it holds that,
\begin{align}
r_{mk+1} 
&\leq \langle \bar{\theta}_{mk+1}-\theta^*, \bar{X}_{mk+1}\rangle \\
&\leq \langle \bar{\theta}_{mk+1}-\theta^*, \bar{\theta}_{mk+1}-\theta^*\rangle^{1/2}\langle \bar{X}_{mk+1}, \bar{X}_{mk+1}\rangle^{1/2}\\
&\leq \norm{\bar{\theta}_{mk+1}-\theta^*}_{\zeta^2(\lambda I + V_{\max\{0,m-1\}})}\norm{\bar{X}_{mk+1}}_{\zeta^2(\lambda I + V_{\max\{0,m-1\}})^{-1}}\\
&\leq \sqrt{b_n^{\delta}}\times \zeta\lambda^{-1/2}\norm{\bar{X}_{mk+1}}_2 \label{eq:Xbarlowner}\\
&\leq \zeta\sqrt{\frac{b_n^{\delta}}{\lambda}}
\end{align}
where in much the same way as with  \eqref{eq:Sm_err}, \eqref{eq:Xbarlowner}  follows from noting that,  $V_m$ is positive semi-definite so the matrix order is reversed through inversion, i.e.  
${x}^{\top}(\lambda I + V_m)^{-1} x \leq {x}^{\top} (\lambda I )^{-1} x
$ for all $x \in \R^d$. It follows that 
\begin{align}
\sum_{m=1}^{\lfloor n/k \rfloor-1}r_{mk+1} \leq \sqrt{\frac{n}{k}\sum_{m=1}^{\lfloor n/k \rfloor-1}r^2_{mk+1} } \leq 2\norm{\theta_t}_{\L_{\infty}}\sqrt{\frac{n b_n^{\delta}}{k\lambda }}
\end{align}
which in turn leads to the 
\begin{align}
\mathcal R_{\bar{\boldsymbol{\pi}}}(n)&\leq (1-\delta)k\left(\norm{\theta_t}_{\L_{\infty}}+\sum_{m=1}^{\lfloor n/k \rfloor-1}  \E r_{mk+1}\right)+\delta \norm{\theta_t}_{\L_{\infty}}\\
&\leq \norm{\theta_t}_{\L_{\infty}}\left( (1-\delta)k \left( 1+ 2\sqrt{\frac{n b_n^{\delta}}{k\lambda }} \right)+\delta\right)\label{eq:Rpibar}
\end{align}
 By \eqref{eq:Rpibar} and \eqref{eq:linkingregrets}, taking $\delta = 1/n$, and noting that $\phi_k \leq a e^{-\gamma k}$ for some $a,\gamma \in (0,\infty)$, we have,
\begin{align}
\frac{\mathcal R_{{\boldsymbol{\pi}}}(n)}{\norm{\theta_t}_{\L_{\infty}}} 
&\leq 6n^2 a e^{-\gamma k} + k \left (1+ 4\sqrt{n} \norm{\theta_t}_{\L_{\infty}} + \sqrt{\frac{8 d n\log(n(1+\frac{n}{\lambda d}))}{\lambda}}\right)+\frac{1}{n}\label{eq:fullreg}
\end{align}
Optimizing \eqref{eq:fullreg} for $k$ 
we obtain
\begin{align}\label{eq:kstar}
k^{\star}=\left \lceil \frac{1}{\gamma}\log\left( \frac{6 a \gamma n^{2} }{1+ 4\sqrt{n} \norm{\theta_t}_{\L_{\infty}} + \sqrt{\frac{8 d  n \log(n(1+\frac{n}{\lambda d}))}{\lambda}}} \right)\right \rceil .
\end{align}
For $n \geq \left \lceil \frac{3 a \gamma \sqrt{\lambda}}{2 \sqrt{\lambda} \norm{\theta_t}_{\L_{\infty}}+\sqrt{2}} \right \rceil $ and $k=\max\{1,k^{\star}\}$ we have,
\begin{align}
 \frac{\mathcal R_{{\boldsymbol{\pi}}}(n) }{\norm{\theta_t}_{\L_{\infty}}}
&\leq \frac{4}{\gamma}\sqrt{n d\log(n(1+\frac{n}{\lambda d})}\left (1+ 4\norm{\theta_t}_{\L_{\infty}} + \sqrt{\frac{1}{2\lambda}}\right)\log\left( \frac{3 a \gamma n^{2} }{ 2\norm{\theta_t}_{\L_{\infty}} + \sqrt{\frac{2}{\lambda}}} \right)+\frac{1}{n}\\
& \leq  \left(\frac{12(\sqrt{2\lambda}+4\sqrt{2\lambda}\norm{\theta_t}_{\L_{\infty}}+1)}{\gamma\sqrt{2\lambda} }\right) \sqrt{2 d   n\log(n(1+\frac{n}{\lambda d}))} \log (n) +\frac{1}{n}.
\end{align}
\end{proof}
Algorithm~\ref{alg:fh} can be turned into an infinite-horizon strategy using a classical doubling-trick. 
 The procedure is outlined in Algorithm~\ref{alg:ih} below. 
 As in the finite-horizon setting, the algorithm aims to minimize the regret with respect to \eqref{eq:tildenu}, in the case where the $\varphi$-mixing coefficients of the process $(\theta_t,~ t \in \N)$ satisfy $\varphi_m \leq a e^{-\gamma m}$ for some fixed $a, \gamma \in (0,\infty)$ and all $m \in \N$. The algorithm works as follows.
At every round $i =0,1,2,\dots$ a horizon is determined as $n_i = 2^i n_0$ with $n_0:= \max\left \{1, \left \lceil \frac{3 a \gamma \sqrt{\lambda}}{2 \sqrt{\lambda} \norm{\theta_t}_{\L_{\infty}}+\sqrt{2}} \right \rceil \right \}$ and the algorithm plays the finite-horizon strategy specified in Algorithm~\ref{alg:fh} from $t=\sum_{j=0}^{i-1} n_j$ to $t=\sum_{j=0}^{i} n_j$. 

The regret of this algorithm is given in Theorem~\ref{thm:regih} below. 
 \begin{algorithm}[ht!]
\caption{LinMix-UCB (infinite horizon)}\label{alg:ih}
\begin{algorithmic}
\Require regularization parameter $\lambda$; 
$\varphi$-mixing rate parameters: $a,\gamma \in (0,\infty)$
\State~
\State $n_0 \gets \max\left \{1, \left \lceil \frac{3 a \gamma \sqrt{\lambda}}{2 \sqrt{\lambda} \norm{\theta_t}_{\L_{\infty}}+\sqrt{2}} \right \rceil \right \}$
\State~
\For{$i=0,1,2,
\dots$}
\State $n_i \gets 2^i n_0$
\State Run Algorithm\ref{alg:fh}$(n_i,\lambda, a, \gamma)$ from $t = (2^i-1)n_0+1$ to $t = (2^{i+1}-1)n_0$
\EndFor
\end{algorithmic}
\end{algorithm}
\begin{theorem}\label{thm:regih}
Suppose, as in Theorem~\ref{thm:regfh}, that there exist $a, \gamma \in (0,\infty)$ such that 
$\varphi$-mixing coefficients of the stationary process $(\theta_t,~ t \in \N)$ satisfy 
$\varphi_m \leq a e^{-\gamma m}$ for all $m \in \N$. 
Then, the regret (with respect to $\tilde{\nu}_n)$ of LinMix-UCB (infinite horizon) after $n$ rounds of play satisfies
\begin{align*}
\frac{\mathcal R_{{\boldsymbol{\pi}}}(n)}{2\norm{\theta_t}_{\L_{\infty}}}
\leq \left(n_0+C(\log_2 (n+1)+1)\log (2(n+1))\sqrt{ (n+1)d\times\log\left (2(n+1)(1+\frac{2(n+1)}{\lambda d})\right)} \right)
\end{align*}
where $C:= \left (\frac{12(\sqrt{2\lambda}+4\sqrt{2\lambda}\norm{\theta_t}_{\L_{\infty}}+1)}{\gamma\sqrt{2\lambda} }\right) $, and $n_0 := \max\left \{1, \left \lceil \frac{3 a \gamma \sqrt{\lambda}}{2 \sqrt{\lambda} \norm{\theta_t}_{\L_{\infty}}+\sqrt{2}} \right \rceil \right \}$, and $\lambda >0$ is the regularization parameter.
\end{theorem}
 \begin{proof}
 For $n \in \N$, let $j(n) := \min\{i \in  \N: \sum_{i=0}^i n_i \geq n\}$.
Recall that the algorithm plays 
the finite-horizon strategy of Algorithm~\ref{alg:fh} during
non-overlapping intervals of length $n_i = 2^i n_0,~i=0,1,2,\dots$ with $n_0 := \max\left \{1, \left \lceil \frac{3 a \gamma \sqrt{\lambda}}{2 \sqrt{\lambda} \norm{\theta_t}_{\L_{\infty}}+\sqrt{2}} \right \rceil \right \}$. By Theorem~\ref{thm:regfh} and that $\sum_{i=0}^{j(n)}\frac{1}{n_i} \leq n_0\sum_{i=0}^{\infty}2^{-i}\leq 2n_0$,
\begin{align}\label{eq:regsum}
 \frac{\mathcal R_{{\boldsymbol{\pi}}}(n) }{\norm{\theta_t}_{\L_{\infty}}}
&\leq 2 n_0+ C \sum_{i=0}^{j(n)}\log(n_i)\sqrt{2d n_i\log(n_i(1+\frac{n_i}{\lambda d}))}
\end{align}
with the constant $C$ as given in the theorem statement. 
The result follows from \eqref{eq:regsum}, and the fact that  as follows from the definition of $j(n)$ we have,
$\sum_{i=0}^{j(n)} n_i = n_0(2^{j(n)+1}-1) \geq n$, so that 
$j(n) = \left \lceil \log_2 (\frac{n}{n_0}+1)\right \rceil$ and $2^{j(n)} \leq 2(n+1)$. 
\end{proof}
\section{\sc \bfseries Outlook}
We have formulated a generalization of both the classical linear bandits with iid noise, and the finite-armed restless bandits. In the problem that we have considered, an unknown $\R^d$-valued stationary $\varphi$-mixing sequence of parameters $(\theta_t,~t \in \N)$ gives rise to the pay-offs. We have provided an approximation of the optimal restless linear bandit strategy via a UCB-type algorithm, in the case where the process $(\theta_t,~t \in \N)$ has an exponential mixing rate. The regret of the proposed algorithm, namely LinMix-UCB, with respect a more relaxed oracle which always plays a multiple of  $\E\theta_t$, is shown to be $\O\left(\sqrt{d n\polylog(n) }\right)$. This result differs from that of \cite{GK19} for the (simpler) finite-armed restless $\varphi$-mixing bandits, in that they do not require an exponential $\varphi$-mixing rate in order to ensure an $\O(\log n)$ regret with respect to the highest stationary mean in their setting. 
Our algorithm requires the mixing rate parameters $a, \gamma $. While this is  standard in time-series analysis, an interesting objective would be to relax this condition and infer the mixing parameters while maximizing the expected cumulative pay-off.
Observe that estimating the mixing coefficients while playing a bandit strategy requires more care as the estimate can not be readily obtained by methods such as those proposed by \cite{KL23, GK24} which are designed to estimate mixing coefficients from a fully observed sample-path of a stationary process. Note that when playing a bandit strategy, the learner only partially observes the process $(\theta_t, t \in \N)$. 
Moreover, a key challenge in estimating $\varphi$-mixing coefficients, even from full observations, lies in conditioning on potentially rare events with small probabilities. At this point, it is unclear whether $\varphi$-mixing coefficients can indeed be consistently estimated from stationary sample-paths. 
Another natural open problem is the derivation of a regret lower-bound with respect to $\nu_n$.

\vskip 0.2in


\end{document}